\documentclass[11pt]{article}
\usepackage[a4paper, margin=1in]{geometry}
\usepackage{hyperref}
\usepackage{url}
\usepackage{amsmath,amssymb,amsthm}
\usepackage{algorithm,algorithmic}
\usepackage{dsfont}
\usepackage{multirow}
\usepackage{authblk}
\newtheorem{theorem}{Theorem}
\newtheorem{lemma}{Lemma}
\newtheorem{corollary}{Corollary}
\newtheorem{definition}{Definition}

\title{On Fundamental Limits of Robust Learning}
\author{Jiashi Feng}
\affil{Department of ECE, National University of Singapore}
\date{}
\begin{document}

\maketitle

\begin{abstract}
We consider the problems of robust PAC learning from distributed and streaming data, which may contain malicious errors and outliers, and analyze their fundamental complexity questions. 	In particular, we establish lower bounds on the communication complexity for distributed robust learning performed on multiple machines, and on the space complexity for robust learning from streaming data on a single machine. These results demonstrate that gaining robustness of learning algorithms is  usually at the expense of increased complexities. As far as we know, this work gives the first complexity results for distributed and online robust PAC learning.
\end{abstract}

\section{Introduction}
The last decade has witnessed a tremendous growth in the amount of data involved in machine learning tasks. In many cases, data volume has outgrown the capacity of memory of a single machine and it is increasingly common that learning tasks are performed in a \emph{distributed}   fashion on many machines \cite{balcan2012distributed,daume2012efficient,zhang2013information} or an  \emph{online} fashion on a single machine \cite{muthukrishnan2005data}.

This gives rise to the following fundamental, yet rarely investigated, complexity questions for distributed and online learning algorithms. Considering a concrete example (our analysis is beyond this case): suppose one has the positive and negative examples stored on two separate machines, then how much communication is necessary to learn a good hypothesis to specified error rate? and how about for the case where more machines are involved? In addition, if the positive and negative examples arrives in a single machine with arbitrary order, how much memory must be used to store the necessary information to learn a good hypothesis? if multiple passes of the data are allowed, how does the memory cost scale then?


In this work, we investigate the above communication and space complexities of distributed and online probably approximately correct (PAC) learning \cite{valiant1984theory}, {\em in the presence of malicious errors (outliers) in the training examples}. PAC learning with malicious errors  was first proposed in \cite{valiant1985learning} and then generalized in \cite{kearns1993learning}. In this learning setting, there is a fixed probability $ \lambda $ of an error occurring independently on each request for a training example. This error may be of an arbitrary nature~--~in particular, it may be chosen by an adversary with unbounded computational resources, and exact knowledge of the target representation, the target distributions, and the current internal state of the learning algorithm. Though various robust learning algorithms have been proposed in literature to tackle the outliers in the data, the fundamental limits of these algorithms as well as of the problem itself, in terms of their communication and space complexities, are rarely investigated.


\subsection{Our Contributions}
We analyze and provide general lower bounds on the amount of communication and memory cost needed to learn a good hypothesis, from either distributed or streaming data, in the presence of malicious errors and outliers.

Our results reveal the connection between the communication complexity and the VC-dimension  $ d $ of the hypothesis to learn, as well as the outlier probability $ \lambda $. We first derive a lower bound of $ \Omega(\frac{1-\epsilon}{1-\lambda}d )$ with $ \epsilon $ the specified error rate, for the simplest distributed learning protocol~--~only two machines are involved and only a single message is allowed to communicate. We then extend the  lower bound analysis to more general distributed learning protocols, allowing $ k $-machine and $ t $-round communications, and obtain a general lower bound of $ \Omega(\frac{1-\epsilon}{1-\lambda}\frac{dk}{t^2} )$. All these lower bounds  present an additional factor proportional to $ \lambda $, which explains the extra communication expense brought by the outliers in the learning process.

We also analyze the $r$-pass online robust PAC learning algorithms. We demonstrate their space complexity lower bounds can be conveniently deduced from the communication complexity results for distributed learning algorithms and hence obtain a space complexity lower bound of $\Omega(\frac{1-\epsilon}{1-\lambda}\frac{d}{r} )$.

Finally, we provide a communication efficient distributed robust PAC learning algorithm along with its  performance guarantee. Experimental studies on synthetic and real data demonstrate the proposed algorithm consumes significant less communication to achieve comparable accuracy with a naive distributed learning algorithm, which aggregates all the examples to a single machine.

\subsection{Related Works}
Since its introduction in \cite{yao1979some}, communication complexity \cite{communicationKush} has proven to be a powerful technique for establishing lower bounds in a variety of settings, including distributed \cite{zhang2013information,duchi2014information,garg2014communication} and streaming data models \cite{muthukrishnan2005data}.

Several recent works on communication complexity for distributed PAC learning include \cite{balcan2012distributed,daume2012efficient,daume2012protocols} and there are also some works about distributed statistical estimation \cite{zhang2013information,duchi2014information}. In particular, Duchi \emph{et al.} \cite{duchi2013distance} demonstrated an analysis tool based on the information Fano inequalities. Garg \emph{et al.} \cite{garg2014communication} investigated  how the communication cost scales in the parameter dimensionality for distributed statistical estimation. Kremer \emph{et al.} \cite{kremer1999randomized} provide a connection between communication complexity and VC-dimension. However, those works generally focus on the case no outliers or malicious errors are presented in the learning process.

On the other hand, studies on streaming algorithms focus on the scenario where input arrives very rapidly and there is limited memory to store the input \cite{muthukrishnan2005data}, and investigate the necessary space cost under the most adversarial data orders for  various data mining problems, include quantile query \cite{zhang2013information}, frequent elements query \cite{metwally2005efficient,alon1996space}, regression \cite{wong2005fast} and model monitoring \cite{korn2003checks}. Those algorithms  do not consider the case with malicious errors either.

\section{Problem Setup and Preliminaries}
\paragraph{Robust PAC Learning} In a classical distribution-free learning setting \cite{valiant1984theory}, the probably approximately correct (PAC) learning model typically assumes the oracles  {\sc Pos}  and {\sc Neg}  always \emph{faithfully} return positive and negative examples from a sample domain $ X $, drawn according to target distributions respectively. However, in many real environments, it is possible that an erroneous or even adversarial example is given to the learning algorithm. We consider in this work such a generalized  PAC learning problem,  termed as \emph{robust PAC learning}, where malicious outliers are possibly presented in training examples \cite{kearns1993learning}.
%

In particular, we consider following erroneous oracles in robust PAC learning. For a  error probability $ 0 \leq \lambda <1/2 $, we have access to following two noisy oracles -- $ \text{\sc Pos}^\lambda $ and $ \text{\sc Neg}^\lambda $ -- which behave as follows: when oracle $ \text{\sc Pos}^\lambda $  is called, with probability $ 1-\lambda $, it returns positive examples as in the error free model. But with probability $ \lambda $, an example about which absolutely no assumptions can be made is returned. In particular, this example may be dynamically and maliciously chosen by an adversary. The oracle  $ \text{\sc Neg}^\lambda $ behaves similarly.

Formally, the task of a robust PAC learning algorithm is to find a hypothesis $ h \in \mathcal{H} $,  with access only to the above noisy oracles $ \text{\sc Pos}^\lambda $ and $ \text{\sc Neg}^\lambda $ and with predefined parameters $ \epsilon, \delta $ and $ \lambda $, such that: for any input values $ 0<\epsilon, \delta <1 $ and $ \lambda < 1/2 $, the output hypothesis $ h $   has bounded errors $ e(h) < \epsilon $, with probability at least $ 1-\delta $. Throughout the paper, we assume a constant failure probability  $ \delta $ for the simplicity of analysis. Note that we always have $ \lambda < \epsilon $ as proved in \cite{kearns1993learning}.

\paragraph{Distributed and Online PAC Learning} As mentioned above, in this work we consider a realistic case where the entire data set is too large to store on a single machine. Thus  following two learning frameworks become the choices to deal with the issues brought by the large scale of the dataset, which include:
\begin{enumerate}
	\item \emph{Distributed robust learning}. Data are distributed on $ k $ machines. The problem  of interest is how to communicate among the $ k $ machines, and especially how much  communication is necessary to produce a low-error hypothesis. Namely, we are concerned about the \emph{communication complexity} of distributed robust learning.	
	\item \emph{Online robust learning}. Data are revealed sequentially (in multiple passes) to a single machine. The problem we want to answer  is how much memory cost is necessary to maintain the intermediate information, in order to obtain the low-error hypothesis, \emph{i.e.}, the \emph{space complexity} of online robust learning.
\end{enumerate}
We regard communication (the number of bits communicated) and memory as a limited resource, and therefore we not only want to optimize the hypothesis error $ e(h) $ but also the resource cost of the whole learning procedure. We aim at providing a fundamental understanding on the necessary communication and space cost in order to learn a hypothesis $ h $ with a specified error $ \epsilon $, in presence of a constant outlier fraction $ \lambda $.


\paragraph{Communication Protocols and Complexity} 
In the above distributed robust learning setting, several different  communication protocols among the machines can be employed. For instance, in a \emph{$1$-round $2$-machine} public-randomness protocol between two machines $ M_1 $ and  $ M_2 $, denoted as $ P^{M_1\rightarrow M_2} $, $ M_1 $ is only allowed to send a single message to $ M_2 $ which must then be able to output the hypothesis $ h $. Similarly, when multiple-round communications are allowed between $ M_1 $ and $ M_2 $, \emph{i.e.}, multiple messages can be communicated among the two machines, we are using a  \emph{$t$-round $2$-machine} protocol, which is denoted as $ P^{M_1 {\leftrightarrow} M_2} $. Extending the above  protocols to the cases involving $k$ machines gives the \emph{$1$-round $k$-machine} $ P^{M_1  \rightarrow \ldots \rightarrow M_k} $ and the \emph{$t$-round $k$-machine} $ P^{M_1 {\leftrightarrow}  \ldots {\leftrightarrow} M_k} $ protocols respectively.

The communication complexity of a  protocol is the minimal number of bits needed to exchange between machines for learning the hypothesis $ h $.
Assume a specified error rate $  0 < \epsilon < 1 $ of the learned hypothesis $ h $, we specifically consider following  communication  complexities $ R_\epsilon^{M_1\rightarrow M_2}(h) $ -- the communication complexity of randomized $1$-round  $2$-machine  protocol, and $ R_\epsilon^{M_1  {\leftrightarrow} \ldots {\leftrightarrow} M_k}(h) $ -- the complexity of the  $ P^{M_1 {\leftrightarrow}  \ldots {\leftrightarrow} M_k} $  protocol.

We start the communication complexity analysis  from the  case when there are two machines, and the $1$-round protocol  $ P^{M_1\rightarrow M_2} $ is employed. After demonstrating the lower bound on the complexity  $ R_\epsilon^{M_1\rightarrow M_2}(h) $, we proceed to analyze the complexity  $ R_\epsilon^{M_1  {\leftrightarrow} \ldots {\leftrightarrow} M_k}(h) $. The complexity lower bounds for other protocols, such as $ 2 $-machine $t$-round and $ k $-machine $ 1 $-round protocols, are provided in the appendix. 


\section{Main Results on Distributed Robust PAC Learning}
In this section we analyze the lower bounds on the communication cost for distributed robust PAC learning. We then extend the results to an online robust PAC learning setting in Section~\ref{sec.online}, through a one-direction equivalent lemma between  distributed and online learning, in terms of communication complexity and space complexity.

\subsection{Communication Models }

%

For distributed learning, popular communication models  include:
(1) \emph{Blackboard model}:  any message sent by a machine is written on a blackboard visible to all machines; (2) \emph{Message passing model}:  a machine $ p_i $ sending a message specifies another machine $ p_j $ that will receive this message; (3) \emph{Coordinator model}:  there is an additional machine called the {\em coordinator}, who receives no input. Machines can only communicate with the coordinator, and not with each other directly.

We will focus on the message-passing model and the coordinator model, considering the blackboard model may introduce extra communication overload and is not so practical. Note that the coordinator model is almost equivalent to the message-passing model, up to a  multiplicative factor of $2$, since instead of machine $ i $ sending the message  to machine $ j $, machine $ i $ can transmit message  to the coordinator, and the coordinator forwards it to machine $ j $. Therefore, in the following sections, we mainly provide the complexity results for the message passing model, which can then be used to deduce the complexities of the coordinate model easily.

\subsection{Lower Bound for $ 1 $-Round $ 2 $-Party Communication}
From now on, we derive the main  results of this work~--~the communication complexity lower bounds for distributed robust PAC learning: how much information communication is needed to learn a hypothesis $h$ with specified error rate~$\epsilon$. Since we are mainly concerned about how communication complexity scales with $ \epsilon $, outlier fraction $\lambda$ and VC-dimension of the hypothesis $ d $, in developing the results, we assume the failure probability  $ \delta $ is fixed as a constant, such as $ 0.1 $.

We  start with one simplest distributed learning setting: only two machines are involved and only a single message communication between them is allowed. We first construct a connection between the communication complexity and VC-dimension \cite{vapnik1971uniform} of the hypothesis to learn, for this $ 1 $-round $ 2 $-machine communication protocol. 

The result is basically obtained by following Yao's principle \cite{yao1977probabilistic} (Theorem \ref{theo:yao} in the appendix) to investigate the communication complexity for a ``difficult'' distribution of the examples over the two machines, which then provides a lower communication complexity bound for  general example distributions over the machines. The lower bound on the communication complexity for the specific ``difficult'' sample distribution is constructed by limiting the information capacity (depending on the VC-dimension $ d $) of the communication channel between the two machines, and we show that with such limited capacity, the inevitable encoding error will result in a lowered bounded error rate $ \epsilon $ of the learned hypothesis $h$. Based on the above technique, we obtain the lower bound on the complexity communication of the protocol $ P^{M_1\rightarrow M_2} $. More details on the  tools and proofs for the following claim are provided  in the appendix.


\begin{theorem}[Communication Complexity Lower Bound for  $1$-Round $2  $-Machine Protocol]
	\label{theo:lb_12}	
	For learning a hypothesis $ h $ with a constant error $ \epsilon $, where the oracles have an outlier rate of $ 0\leq \lambda < \epsilon $, we have the following communication complexity lower bound for $ 1 $-round and $ 2$-machine communication protocol,
	\begin{equation*}
	R_\epsilon^{M_1\rightarrow M_2}(h) = \Omega\left(\frac{1-\epsilon}{1-\lambda} d\right),
	\end{equation*}
	for a constant success probability. Here  $ d $ is the VC-dimension of the hypothesis $ h $.
\end{theorem}
From the above lower bound, one can observe that the necessary communicated bits between two machines is  monotonically increasing with the outlier rate $ \lambda $ when the error rate $ \epsilon $ is fixed, while monotonically decreasing with the specified error rate $ \epsilon $. When there are no outliers in the samples, \emph{i.e.}, $ \lambda=0 $, the lower bound for learning a hypothesis with $ e(h)<\epsilon $ reduces to $ \Omega((1-\epsilon)d) $, which matches the result provided in \cite{balcan2012distributed} for faithful oracles.


Applying Theorem \ref{theo:lb_12} directly provides a communication lower bound on learning the  linear half spaces in the real space $ \mathbb{R}^p $,  whose VC-dimension is known to be $ (p+1) $.
\begin{corollary}
	For the hypothesis $ \hbar $ of linear half spaces in $ \mathbb{R}^p $, applying the above theorem gives the one-round communication complexity between two machines to learn a half space is lower bounded as
	\begin{equation*}
	R_{\epsilon}^{M_1 \rightarrow M_2}(\hbar) = \Omega\left(\frac{1-\epsilon}{1-\lambda}  (p+1)\right),
	\end{equation*}	
	for a constant success probability.
\end{corollary}
In Section \ref{sec:ws}, we compare the above  lower bound for half space learning with the bound given in \cite{daume2012efficient}, and we propose a communication-efficient distributed half-space learning algorithm, whose communication complexity is close to the above lower bound.

\subsection{Lower Bound for $t$-Round  $ 2 $-Machine Communication}
We then consider a more complicated and practical communication protocol, which allows the two machines to communicate $ t > 1$ messages in turn. Such a protocol is called a $t$-Round  $ 2 $-Machine protocol. Here we assume, by convention, that one machine sends the first message and the recipient machine of the last message announces the learned hypothesis $ h $. To lower bound the  complexity of multiple round communications, we need to investigate the size of the message  sent in each round and invoke a technique of round elimination. Now we elaborate on how we can obtain the desired complexity bound.

We first provide a more detailed description on the multiple round communication protocol, which explicitly specifies the size of message sent in each round.
\begin{definition}
	A $ [t;\ell_1,\ldots,\ell_t]^{M_1} $ (resp. $ [t;\ell_1,\ldots,\ell_t]^{M_2} $) communication protocol is the protocol where the machine $ M_1 $ (resp. $ M_2 $) starts the communication, the size of the $ i $th message is $ \ell_i $ bits, and the communication goes on for $ t $ rounds.
\end{definition}
The communication complexity for such a protocol is then proportional to $ \sum_{i=1}^t \ell_i $. We employ a \emph{round elimination} technique \cite{sen2003lower} to deduce the communication complexity lower bound for a $ t $-round protocol from a $ (t-1) $-round protocol. After applying the round elimination for $ (t-1) $ times, it is shown that we only need to bound the complexity of a $ 1 $-round protocol, which has been provided in Theorem \ref{theo:lb_12}.

The round elimination, which is formally described by the following lemma, is based on the intuition that the existence of a ``good''  $ t $ round protocol  with $ M_1 $ starting, implied that there exists a ``good'' $ (t-1) $-round protocol where $ M_2 $ communicates first.
\begin{lemma}[Round Elimination, adapted from  \cite{sen2003lower}]
	\label{lemma:RE}
	Assume $ h $ is a hypothesis to learn. Suppose the communication between two machines has a $ [t;\ell_1,\cdots,\ell_t]^{M_1} $ protocol which produces a hypothesis with error less than $ \epsilon $. Then there is a $ [t-1;\ell_1+\ell_2,\ell_3,\cdots,\ell_t]^{M_2} $  protocol for learning $ h $ with error less than $ \epsilon + \sqrt{\ell_1/d} $. Here $ d $ is also the VC-dimension of $ h $.
\end{lemma}

%
%

Applying the above round elimination lemma,  we can reduce any $ t $-round communication protocol to a $ 1 $-round communication protocol. Thus applying the communication complexity for $ 1 $-round protocol in Theorem \ref{theo:lb_12} gives the following lower bound.

\begin{theorem} [Communication Complexity Lower Bound for  $ t $-Round $ 2 $-Machine Protocols]
	\label{theo:t-round}
	Let $ \epsilon $ be the specified error rate of $ h $. Let $ \lambda $ be the outlier rate in the oracles with $ \lambda < \epsilon $.
	Consider a $ t $-round  communication protocol to learn hypothesis $ h $.  Then the communication complexity is lower bounded as
	\begin{equation*}
	R_\epsilon^{M_1 \leftrightarrow M_2} = \Omega\left(\frac{1-\epsilon}{1-\lambda}  \frac{d}{t^2}\right),
	\end{equation*}
	for a constant success probability. Here $ d $ is the VC-dimension of $ h $.
\end{theorem}

\subsection{Lower Bound for $ t $-Round  $ k $-Machine Communication}
\label{sec:t-round-k-machine}
We now proceed to obtain the communication complexity lower bound for a more general  protocol $ P^{M_1\rightarrow \ldots \rightarrow M_k} $. To extend the above lower bound for $ 2 $-machine communication to the $ k $-machine case, we use the symmetrization technique \cite{phillips2012lower} to reduce  $ k $-machine communication to $ 2 $-machine communication. The symmetrization constructs a $ 2 $-machine protocol $ P^{M_1 \rightarrow M_2} $ from a $ k $-machine protocol $ P^{M_1 \rightarrow \ldots \rightarrow M_k} $ with communication cost $ \mathbb{E}[R_\epsilon^{M_1 \rightarrow M_2} ]=2R_\epsilon^{M_1 \rightarrow \ldots \rightarrow M_k} /k $, where the expectation is taken on the randomness of the samples returned by oracles. Then we can derive the  bound for $ R_\epsilon^{M_1 \rightarrow \ldots \rightarrow M_k} $ from $R_\epsilon^{M_1 \rightarrow M_2} $.  The result is formally stated in the following theorem, whose proofs are given in the appendix.

\begin{theorem}[Communication Complexity Lower Bound for  $ 1 $-Round $ k $-Machine Protocols]
	\label{theo:k-machine}
	Let $ \epsilon $ be the specified error rate of learned hypothesis $ h $.	Let $ \lambda $ be the outlier rate in the oracles with $ \lambda < \epsilon $. 	For  $ 1 $-round $ k $-machine protocol, its communication complexity is lower bounded as
		\begin{equation*}
		R_\epsilon^{M_1\rightarrow \ldots \rightarrow M_k} = \Omega \left(\frac{1-2\epsilon}{1-\lambda} d k \right),
		\end{equation*}
		for a constant success probability. Here $ d $ is the minimal VC-dimension of the hypothesis $ h $.
\end{theorem}
In the above theorem, we observe that there is a factor of $ 2\epsilon $ instead of $ \epsilon $, which appears due to the employed the symmetrization technique and introduces a gap with the lower bound for $ 2 $-machine protocol. In the future, we will further narrow such a gap.
 
Then similar to obtaining the results in Theorem \ref{theo:t-round} and Theorem \ref{theo:k-machine}, applying the  round elimination and  machine elimination by symmetrization together gives a  general  lower bound on the communication complexity for the $ t $-round $ k $-machine protocol, as stated in the following theorem.

\begin{theorem}[Communication Complexity Lower Bound for  $ t $-Round $ k $-Machine Protocols]
	\label{coro:k-machine-t-round}
	Let $ \epsilon $ be the specified error rate of learned $ h $. Let $ \lambda $ be the outlier rate in the oracles with $ \lambda < \epsilon $. Then for a $ t $-round $k$-machine protocol, its communication complexity is lower bounded as 
	\begin{equation*}
	 R_\epsilon^{M_1\leftrightarrow M_2 \leftrightarrow \ldots \leftrightarrow M_k}= \Omega \left(\frac{1-2\epsilon}{1-\lambda}\frac{d  k}{t^2}\right),
	\end{equation*}
	for a constant success probability. Here $ d $ is the minimal VC-dimension of the hypothesis $ h $.
\end{theorem}

\subsection{A Communication-Efficient Protocol: Weighted Sampling}
\label{sec:ws}
We present an efficient distributed PAC learning algorithm, focusing on the $ 2 $-machine case, where the machines are denoted as $ A $ and $ B $. The proposed algorithm is mainly based on communicating learned hypotheses and difficult examples, similar to boosting: each hypothesis learned by $ A $ is communicated to and tested by $ B $. Then the difficult examples of $ B $ are randomly sampled by a probability, proportional to how many times the hypothesis fails on them, and communicated back to $ A $ for updating hypothesis. This two-machine protocol algorithm is developed based on the algorithm proposed in \cite{daume2012efficient}.

Following lemma \cite{daume2012efficient} says the random weighted sampling can always find a good representing subset of the entire set on machine $ B $.
\begin{lemma}
	\label{lemma:wsample}
	Let $ B $ have a weighted set of points $ X_B \subset \mathbb{R}^p$ with weight function $ w: X_B \rightarrow \mathbb{R}^+ $. For any constant $ 0<c<1 $, machine $ B $ can send a set $ R_B$ of size $ O(p/c) $  such that any hypothesis that correctly classifies all points in $ R_B $ will mis-classify points in $ X_B $ with a total weight at most $ c \sum_{x\in X_B} w(x) $. The set $ R_B $ can be constructed by a weighted random sample from $ (X_B,w) $ as in Alg. \ref{alg:mwu}, which succeeds with constant probability.
\end{lemma}

\begin{algorithm}
	\caption{Weighted Sampling ({\sc Ws})}
	\label{alg:ws}
	\begin{algorithmic}
		\STATE \textbf{Input:} $ X_A, X_B \subset \mathbb{R}^p$, parameters: $ 0 < \epsilon < 1 $, outlier fraction $ 0 \leq \lambda<0.5 $.
		\STATE \textbf{Initialization:} $ R_B = \emptyset, w_i^0=1, \forall x_i \in X_B, \rho=0.75, c=0.2 (1-\lambda)/(1-2\lambda)^2$.
		\FOR{$t=1,\dots,T=5\log_2(1/\epsilon) $}
		\STATE $ X_A = X_A \cup R_B $;
		\STATE $ h_A^t := \text{Learn}(X_A) $;
		\STATE send $ h_A^t $ to $ B $;
		\STATE $ R_B :=$ {\sc mwu}$ (X_B,h_A^t,\rho, c ) $ (see Alg.\ref{alg:mwu});
		\STATE send $ R_B $ to $ A $;
		\ENDFOR
		\STATE $ h_{AB} = \mathrm{Majority}(h^1_A,h^2_A,\dots,h^T_A) $.
		\STATE \textbf{Output:} $ h_{AB} $ (hypothesis with $ \epsilon $-error on $ X_A \cup X_B $).
	\end{algorithmic}
\end{algorithm}

\begin{algorithm}
	\caption{ \sc{mwu}$(X_B,h_A^t,\rho,c) $}
	\label{alg:mwu}
	\begin{algorithmic}
		\STATE \textbf{Input:} $ h_A^t, X_B \subset \mathbb{R}^p $, parameters: $ 0 < \rho , c <1 $.
		\FORALL{$x_i \in X_B $}
		\IF{ $ h_A^t(x_i)\neq y_i $}
		\STATE $ w_i^{t+1} = w_i^t(1+\rho) $;
		\ELSE
		\STATE $ w_i^{t+1} = w_i^t $;
		\ENDIF
		\ENDFOR
		\STATE randomly sample $ R_B $ from $ X_B $ according to $ w^{t+1} $, whose size $ |R_B| = \min\{p/c\log(p/c),p/c^2\} $.
		\STATE \textbf{Output:} $ R_B $.
	\end{algorithmic}
\end{algorithm}

Based on Lemma \ref{lemma:wsample}, we can obtain the following communication complexity result for Algorithm \ref{alg:ws} 
\begin{theorem}
	\label{theo:alg-ws}
	The two-machine two-way protocol Weighted Sampling in Alg. \ref{alg:ws} for linear separator in $ \mathbb{R}^p $ mis-classifies at most $ \epsilon|D| $ points after $ T = O(\log(1/\epsilon)) $ rounds and uses $ O(\frac{1-\lambda}{(1-2\lambda)^2}p^2\log(1/\epsilon)) $ bits of communication.
\end{theorem}

	Comparing the communication complexity of {\sc Ws} in Theorem \ref{theo:alg-ws} with the upper bound in Corollary \ref{coro:upper-bd} (where the sample description length is $ b = O(p)$ and $ |\mathcal{H}|=2^p $) provides that {\sc Ws} reduces the communication complexity by a factor of $ O( 1/\log(1/\epsilon)\epsilon )$. This demonstrates {\sc Ws} is much more communication efficient than naively  aggregating sufficiently many samples. 

	Consider that VC-dimension of a linear classifier in $ \mathbb{R}^{p} $ is $ (p+1) $. Applying Theorem \ref{theo:t-round} gives a communication complexity lower bound of learning linear classifiers from two machines to be $ \Omega\left(\frac{1-\epsilon}{1-\lambda} (p+1) \right) $. Omitting constant factors gives a ratio between this theoretical complexity lower bound and the practical  complexity of {\sc Ws} in Theorem \ref{theo:alg-ws} as $ O(p) $. This communication efficiency gap, linearly depending on $ p $, is significant when $ p $ is large. The factor $ p $ comes from the algorithm {\sc Mwu} sampling a  subset $ R_B $ whose size depends on $ p $. This gap also demonstrates there is room to enhance communication efficiency for {\sc Ws} algorithm, if we can construct a data subset from $ B $ whose size independent of $ p $.
This $ 2 $-machine communication protocol for distributed classifier learning can be straightforwardly generalized to $ k $-machine communication, using the \emph{coordinator communication} model as follows. Fixing an arbitrary machine (say machine $ 1 $) as the coordinator for the other $ (k-1) $ machines. Then machine $ 1 $ runs the $ 2 $-machine communication protocol from the perspective of $ A $ and the other machines serve jointly as the second machine $ B $. Each other machine reports the total weight $ w(X_i) $ of their data to machine $ 1 $, who then reports back to each machine what fraction of the total weight $ w(X_i)/w(X) $ they own. Then each machine sends the coordinator a random sample of size $ |R_B|w(D_i)/w(D) $ (for $ |R_B| $ see Alg. \ref{alg:mwu}). Party $ 1 $ learns a classifier on the union of its own sample and the joint samples from other machines, and sends the updated classifier back to all the machines.

\section{Online Robust PAC Learning}\label{sec.online}
We briefly discuss here about the space complexity for online robust PAC learning.
We use our results on communication complexity from the previous sections to derive lower bounds on \emph{space complexity}  for robust PAC learning in online setting.
The space complexity has been extensively investigated in the context of streaming methods \cite{muthukrishnan2005data}. However, specific space complexity lower bound for robust learning is still absent.

Following lemma shows the connection between distributed learning algorithm and online learning algorithm, in particular from the perspective of the connection between the communication complexity and the space complexity bound in an online data model.
\begin{lemma}
	\label{lemma:online-distr-equivalence}
	Suppose that we can learn a hypothesis $ h $ using an online algorithm that has $ s $ bits of working storage and makes $ r $ passes over the samples. Then there is a $ k $-machine distributed algorithm for learning $ h $ that uses $ krs $ bits of communications.
\end{lemma}

The above lemma allows us  to deduce the space lower bound for the online learning, given the communication complexity of distributed learning algorithms. 

\begin{theorem}[Space Complexity Lower Bound for $  r$-Pass Online Learning]
	\label{theo:online_pac}
		 Let $ \epsilon $ be the specified error rate of learned $ h $. Let $ \lambda $ be the outlier rate in the oracles with $ \lambda < \epsilon $. For a $ r $-pass online PAC learning algorithm, its space complexity is lower bounded as
		$\Omega \left(\frac{1-\epsilon}{1-\lambda} \frac{d}{r}\right)$,
		where $ d $ is the VC-dimension of hypothesis $ h $.
\end{theorem}

For an online PAC learning algorithm, in order to learn a hypothesis with an error at most $ \epsilon $, a space at the order of $\frac{1-\epsilon}{1-\lambda} d $ must be maintained. 
However, if we allow the data to be passed to the learning algorithm for $r$ passes, then the lower bound on the space cost can be reduced by a factor of $ r $.

\section{Simulations}
In this section, we present simulation studies on the distributed learning algorithm, {\sc Ws}, in Algorithm \ref{alg:ws}, for finding linear classifier in $ \mathbb{R}^p $ for two-machine and $ k $-machine scenarios. We empirically compare it with a {\sc Naive} approach, which sends all samples from $ (k-1) $ machines to a coordinator machine $ 1 $ and then learns the classifier at the coordinator. For any dataset, this accuracy is the best possible.

For the two compared methods, a linear SVM is used as the underlying classifier. We report their training accuracies and communication costs. The  cost of communicating one sample or one linear classifier in $ \mathbb{R}^p $ is assumed to be $ (p+1) $ ($ p $ for describing the coordinates and $ 1 $ for sign or  offset). For instance, the total communication cost of the {\sc Naive} method is the number of samples sent by $ (k-1) $ machines multiplied by the description length of each sample $ (p+1) $, which is equal to $ \sum_{i=2}^k (p+1)|X_i| $. We always set the communication cost of the naive method as $ 1 $ and report the communication cost ratio of {\sc Ws}  to the {\sc Naive} method.

\begin{table}[h]
	\caption{Classification accuracy  of the two compared methods with different malicious error rates, and communication cost (CC) of {\sc Ws} relative to the Cost of {\sc Naive} method.}
	\label{tab:result}
	\centering
	\begin{tabular}{c|c|c|c|c|c|c}
		\hline
		\multirow{2}{*}{Dataset} &\multicolumn{2}{c|}{Accuracy ($ \lambda =0.1 $)} &\multicolumn{2}{c|}{Accuracy ($ \lambda = 0.2$)}  & \multicolumn{2}{c}{Relative CC: {\sc Ws} / {\sc Naive} } \\
		\cline{2-7}
		& \sc{Naive} & \sc{Ws}  & \sc{Naive} & \sc{Ws} & $ \lambda = 0.1 $ & $ \lambda = 0.2$ \\
		\hline
		\hline
		Syn.1  & $85.63$ &$84.97$ & $75.52$ &$71.26$ &$0.42$ &$0.64$ \\
		Syn.2  & $85.23$ &$83.13$ & $75.05$ &$72.33$ &$0.43$ &$0.69$ \\
		Syn.3  & $84.22$ &$80.94$ & $78.52$ &$69.97$ &$0.21$ &$0.34$ \\
		Syn.4  & $83.40$ &$83.61$ & $75.05$ &$71.64$ &$0.20$ &$0.33$ \\
		Cancer  & $84.07$ &$83.03$ & $75.76$ &$73.32$ &$0.32$ &$0.66$ \\
		Mushroom \ & $85.26$ &$83.87$ & $77.39$ &$74.23$ &$0.19$ &$0.42$ \\
		\hline
	\end{tabular}
\end{table}

We report results for two-machine and four-machine protocols on both synthetic and real-world datasets. Four datasets, two each for two-machine and four-machine cases, are generated synthetically from mixture of two Gaussians $\mathcal{N}(\mu_i,\Sigma_i)$ with $ \mu_1 = - \mu_2 $ and $ \Sigma_1, \Sigma_2 $ two randomly generated diagonal matrices. Each Gaussian is carefully seeded to make sure the generated data from two components are separated well. Algorithms access training examples via noisy oracles $ \text{\sc Pos}^\lambda $ and $ \text{\sc Neg}^\lambda $ with a malicious error rate of $ \lambda = 0.1$ and $ 0.2 $. That is, among the training examples, $ 10\% $ or $ 20\% $ of them come from a noisy distribution $ \mathcal{N}(0,I_p) $, which is significantly different from the above two Gaussian components. In addition, we conduct empirical studies on two real-world datasets from UCI repository\footnote{\url{http://archive.ics.uci.edu/ml/}}, including the \emph{Cancer} and \emph{Mushroom} data sets. Statistics on the used data sets are given in Table \ref{tab:statistics} and corresponding results, including accuracies and \emph{relative} communication costs (regarding the cost of {\sc Naive} as $ 1 $), are shown in Table \ref{tab:result}. Observations on the results demonstrate: (1) the proposed multiplicative weighted sampling algorithm achieves prediction accuracy matching the best results provided by naive algorithm, and (2) increasing $ \lambda $ (error rate of oracles) generally brings extra communication cost to achieve comparable performance with the naive algorithm.

\section{Conclusions}
In this work, we provided several theoretic results regarding fundamental limits of the communication complexity for distributed PAC learning and space complexity for online PAC learning, {\em in the presence of malicious errors in the training examples}. We demonstrated how the complexities increase along with the malicious error rates for various distributed learning settings, from simplest two-machine one-round communication protocols to the general multi-machine multi-round communication protocols. A connection between online learning and distributed learning was presented, which gives the space complexity for various online learning protocols. We also provided a boosting flavor distributed robust linear classifier learning algorithm {\sc Ws}, which presented significantly higher communication efficiency than the naive  distributed learning algorithm with negligible classification accuracy loss.

\appendix
\section{Simulations}
More details on the used datasets for the simulation evaluation are summarized in Table \ref{tab:statistics}.
\begin{table}[h]
	\caption{Summary of the used datasets. For synthetic datasets Syn.3 and Syn.4, four machines (machines) are used. For the other four datasets, two machines (machines) are used.}
	\label{tab:statistics}
	\centering
	\begin{tabular}{c|c|c|c}
		\hline
		Dataset & total \# examples & \# examples per machine &\# dimension $ p $\\
		\hline
		\hline
		Syn. 1  & $ 2 \times 10^4 $ & $ 1\times 10^4 $ &$100 $  \\
		Syn. 2  & $ 4 \times 10^4 $ & $ 2 \times 10^4 $ &$100 $  \\
		Syn. 3 & $ 2 \times 10^4 $ & $ 0.5 \times 10^4 $ &$100 $  \\
		Syn. 4  & $ 4 \times 10^4 $ & $ 1\times 10^4 $ &$100 $\\
		Cancer  &$ 683 $ &$ 342 $ &$ 10 $  \\
		Mushroom  &$ 8{,}124 $ &$ 4{,}062 $ &$ 112 $  \\
		\hline
	\end{tabular}
\end{table}

\section{ Communication Complexity Lower Bounds for Other Communication Protocols }
\subsection{Lower Bound for $ 1 $-Round  $ k $-Machine Communication}
Now we consider a slightly more general communication protocol than the $ 1 $-round $ 2 $-machine one, where $ k $ machines are involved, and only one-round communication is allowed for two neighboring machines in the sequence of communication.

We first note that the multi-machine one-round protocol is not more efficient than the two-machine protocol \cite{bar2002information}, in terms of the \emph{maximal size of message } communicated.
The \emph{total} communication complexity lower bound for  $ 1 $-round  $ k $-machine protocol is provided in \ref{sec:t-round-k-machine}.

To show this, we associate with every disjoint sample partition $ I_1,\ldots, I_k $ over $ k $ machines a disjoint sample partition for two machines: for $ j=1,\ldots,k-1 $, $ J_1^j = \bigcup_{i=1}^j I_i $ and $ J_2^j = \bigcup_{i=j+1}^k I_i $, and we have following results.
\begin{lemma}
	For every hypothesis $h: X \rightarrow \{0,1\} $, every disjoint input partition $ I_1,\ldots,I_k $ dividing $ X $ into $ \{X_{I_1},\ldots, X_{I_k}\} $, and every specified error rate $ \epsilon > 0 $,
	\begin{equation*}
	R_\epsilon^{I_1 \rightarrow \cdots \rightarrow I_k, \max}(h) \geq \max_j R^{J_1^j\rightarrow J_2^j}_\epsilon(h),
	\end{equation*}
	where the superscript $ \max $ in the communication complexity denotes the length of the longest message.
\end{lemma}
\begin{proof}
	The two machines can simulate the $ k $-machine protocol as follows: $M_1$ simulates the first $ j $ machines; it then transmits what the $ j $-th machine would have sent to the $ (j+1) $-st machine to $M_2$; $M_2$ completes the computation by simulating the last $ (k-j) $ machines.
\end{proof}
The result also holds for the $ \mu $-distributional complexity for any input distribution $ \mu $. It therefore follows from Theorem \ref{theo:lb_12} that:

\begin{corollary}[Communication Complexity Lower Bound on Maximum Message Size for  $ 1 $-Round $ k $-Machine Protocols]
	Let $ \lambda $ be the outlier rate in the oracles. Let $ \epsilon $ be the specified error rate of $ h $.
	For every hypothesis $ h: X \rightarrow \{0,1\} $, every disjoint input partition $ I_1,\ldots,I_k $, we have
	\begin{equation*}
	R_\epsilon^{ M_1 \rightarrow \cdots \rightarrow M_k, \max}(h) = \Omega\left(\frac{1-\epsilon}{1-\lambda}  d^\prime\right),
	\end{equation*}
	for a constant success probability. Here $ d^\prime $ is the maximal VC-dimension among the hypothesis $ h_{X_{I_j}}  $ learned from $ X_{I_j} $.
\end{corollary}

\section{Proofs of Main Results}
\label{sec:proofs}
\subsection{Tools}

The following essential theorem is a consequence of the minmax theorem in \cite{neumann1928theorie}, which states the randomized public-coin communication complexity is always lower bounded by the distribution complexity.
\begin{theorem}[Yao's principle, \cite{yao1977probabilistic}]
	\label{theo:yao}
	For every function $ f: X \times Y \rightarrow \{0,1\} $, and for every $ 0 < \epsilon < 1 $,
	\begin{equation*}
	R_\epsilon^{A\rightarrow B, \text{pub}}(h) = \max_\mu D_\epsilon^{A\rightarrow B, \mu}(h),
	\end{equation*}
	where $ \mu $ ranges over \emph{all} distributions on $ X \times Y $.
\end{theorem}
This important result provides us with a convenient way to lower bound the communication complexity, through finding a ``hard'' distribution $ \mu $ w.r.t.\ which the communication complexity is easy to evaluate. Throughout the paper, we usually take a product distribution, \emph{i.e.}, $ X $ and $ Y $ are independent of each other, as the hard distribution $ \mu $.

\subsection{Proof of Communication Complexity Lower Bound for  $1$-Round $2$-Machine Protocol, Theorem \ref{theo:lb_12} }
To prove the above lower bound, we need the following lemma, which says a lower bounded decoding error is inevitable when the communication channel capacity is limited. Define $ \mathrm{dist}(z,U) $ to be the minimum hamming distance between $ z \in \{0,1\}^d $ and a vector $ u \in U \subset \{0,1\}^d $.

\begin{lemma}
	\label{lemma:decoding-error}
	Suppose $ d\geq 2 $. For every set $ U \subset \{0,1\}^d $, $ |U| \leq 2^{c d}  $, where $ c = \frac{8(1-\epsilon)}{15(1-\lambda)}  $ with $ 0\leq \lambda < \epsilon \leq 1 $,
	\begin{equation*}
	\mathbb{E}\left[\mathrm{dist}(z,U)\right] \geq \frac{\epsilon-\lambda}{1-\lambda}d,
	\end{equation*}
	where the expectation is taken with respect to the uniform distribution over $ z \in \{0,1\}^d $.
\end{lemma}

\begin{proof}
	For each $ u \in U $, let $ N_u = \{z\in \{0,1\}^d|\mathrm{dist}(z,u) \leq \frac{2(\epsilon-\lambda)}{1-\lambda}d \} $. Denote $ t = \frac{2(\epsilon-\lambda)}{1-\lambda} $, and we have $ 0\leq t < 1 $. For each $ u \in U$,
	\begin{equation*}
	|N_u| = \sum_{i=0}^{td} \binom{d}{i} \leq (ed/(td))^{td} = 2^{td\log(e/t)}.
	\end{equation*}
	Thus,
	\begin{equation*}
	\sum_u |N_u| < 2^{cd+td\log(e/t)}  \leq 2^{d-1}.
	\end{equation*}
	The last inequality is from the assumption that $ d \geq 2$.
	We show that $ \left|\bigcup_u N_u \right| \leq \sum_u \left|N_u\right| \leq 2^{d-1}$, and hence $ \mathbb{E}\left[\mathrm{dist}(z,U)\right] > \frac{\epsilon-\lambda}{1-\lambda}d $ by Markov inequality.
\end{proof}

Now we proceed to prove Theorem \ref{theo:lb_12}.
\begin{proof}
	According to Yao's theorem  (Theorem \ref{theo:yao}), $ R_\epsilon^{A\rightarrow B}(h) = \max_\mu D_\epsilon^{A \rightarrow B, \mu}(h) $, where the maximum is taken over \emph{all} distributions $ \mu $.
	
	
	In order to prove this lower bound, we describe a product distribution $ \mu = \mu_1 \cup \mu_2 $ for which $ D_\epsilon^{A\rightarrow B, \mu}(h) = \Omega(d) $, where $ d$ is the VC-dimension of $h_X $. By definition of the VC-dimension, there exists a set $ S \subseteq Y $ of size $ d $ which is shattered by $ h_X $~--~the hypothesis learned from $ X $. Namely, for every subset $ R \subseteq S $ there exits $ x_R \in X $, such that $ \forall y \in S, h_{x_R}(y)=1 $ iff $ y\in R $. For each $ R \subseteq S $, fix such an $ x_R $. Let $ \mu_1 $ be the uniform distribution over the set of pairs $ \left\{(x_R,y)| R\subseteq S, y \in S \right\} $. Let $ \mu_2 $ be the outliers.
	
	Let $ P $ be a single round deterministic protocol for learning $ h \in \mathcal{H}$ whose cost is \emph{at most} $ c = \frac{8(1-\epsilon)}{15(1-\lambda)}  d $.  Thus, $ P $ induces two mappings. $ P_1: \{0,1\}^d \rightarrow \{0,1\}^c $, determines which $ c $ bits $M_1$ should send to $M_2$ for every given $ x_R $, and $ P_2: \{0,1\}^c \rightarrow \{0,1\}^d $ determines the value of $ h $ computed by $M_2$ for every $ y \in S $, given the $ c $ bits sent by $M_1$. Combining these two mappings together, $ P $ induces a mapping $ P_{1,2} \triangleq P_1 \circ P_2 $ from $ \{0,1\}^d $  into a set $ U \subset \{0,1\}^d $, where $ |U| \leq 2^c $. The expected error of $ h $ is
	\begin{equation*}
	\epsilon(h) = (1-\lambda) \frac{1}{d2^d} \sum_{z\in\{0,1\}^d} \mathrm{dist}\left(z, P_{1,2}(z)\right) + \lambda,
	\end{equation*}
	where $ \mathrm{dist}(\cdot,\cdot) $ denotes the hamming distance between the vectors. Then Lemma \ref{lemma:decoding-error} gives us
	\begin{equation*}
	\epsilon(h) \geq (1-\lambda) \frac{1}{d2^d} \sum_{z\in\{0,1\}^d} \frac{\epsilon - \lambda}{1-\lambda} d + \lambda = \epsilon.
	\end{equation*}
	Therefore, if the communication complexity is limited to less than $ c = \frac{8(1-\epsilon)}{15(1-\lambda)}  d $, the error of the learned hypothesis will be always greater than specified error $ \epsilon $.  Thus, we get the lower bound  in Theorem \ref{theo:lb_12}.
\end{proof}

\subsection{Proof of Communication Complexity Lower Bound for  $ t $-Round $ 2 $-Machine Protocols, Theorem \ref{theo:t-round}}
%
\begin{proof}
	Suppose the hypothesis $ h $ can be learned with $ \epsilon $-error  following a $ t $-round randomized $ [t; \ell_1,\ldots, \ell_t ] $ protocol $ P $. Let $ h_t $ denote the learned hypothesis.  Applying the round elimination Lemma \ref{lemma:RE} to $ P $ repeatedly for $ (t-1) $ times, we see that $ h_1 $ has a  protocol $ Q $ with communication complexity $ c(Q) \leq  \ell_1 + \ldots + \ell_t $ and 
	\begin{eqnarray*}
		&&\mathrm{err}(Q) \\
		&\leq& \epsilon + \sqrt{\frac{\ell_1}{d}} + \sqrt{\frac{\ell_1 + \ell_2}{d}} + \cdots + \sqrt{\frac{\ell_1+\cdots + \ell_{t-1}}{d}} \\
		&\leq& \epsilon + t\sqrt{\frac{\ell_1+\cdots+\ell_t}{d}}.	
	\end{eqnarray*}
	Suppose the communication complexity of $ P $ satisfies $ c(P) \leq  \frac{8(1-\epsilon)}{15(1-\lambda)}  \frac{d}{t^2} $. Then $ \ell_1+\cdots+\ell_t \leq \frac{8(1-\epsilon)}{15(1-\lambda)}  \frac{d}{t^2} $, so $ \mathrm{err}(Q)\leq 2\epsilon  $ with communication cost less than  $\frac{8(1-\epsilon)}{15(1-\lambda)} d$. This is a contradiction to Theorem \ref{theo:lb_12}. Therefore,  the communication complexity is at least  $ \frac{8(1-\epsilon)}{15(1-\lambda)}  \frac{d}{t^2} $.
\end{proof}

\subsection{Proof of Communication Complexity Lower Bound for  $ 1 $-Round $ k $-Machine Protocols, Theorem \ref{theo:k-machine}}

\begin{proof}
	To extend the above lower bound for $ 2 $-machine communication to the $ k $-machine case, we use the symmetrization technique \cite{phillips2012lower} to reduce  $ k $-machine communication to $ 2 $-machine communication.
	
	The symmetrization is conducted as follows. Consider a protocol $ P $ for this $ k $-machine problem, which works on this distribution, communicates $ c(P) $ bits in expectation. We build from $ P $ a new protocol $ P^\prime $ for a $ 2 $-machine problem. In the $ 2 $-machine problem, suppose that $M_1$ gets input $ x $ and $M_2$ gets input $y $, where $ x, y \in \{0,1\}^d $ are independent random vectors. Then $ P^\prime $ works as follows: $M_1$ and $M_2$ randomly choose two distinct index $ i,j \in \{1,\ldots,k\} $ using the public randomness, and they simulate the protocol $ P $, where $M_1$ plays machine $ i $ and lets $ x_i = x $, $M_2$ plays machine $ j $ and lets $ x_j = y $, and they both play all of the rest of the machines; the inputs of the rest of the machines is chosen from shared randomness. $M_1$ and $M_2$ begin simulating the running of $ P $. Every time machine $ i $ should speak, $M_1$ sends to $M_2$ the message that machine $ i $ was supposed to communicate, and vice versa. When any other machine $ p_r $ ($ r\neq i,j $) should speak, both $M_1$ and $M_2$ know his input so they know what it should communicate, thus no communication is actually needed. A key observation is that the inputs of the $ k $ machines are uniform and independent and thus entirely symmetrical, and since the index $ i $ and $ j $ were chosen uniformly at random, then the expected communication performed by the protocol $ P^\prime $ is $ \mathbb{E}\left[c(P^\prime)\right] = 2c(P)/k $. Since $ c(P^\prime) \geq \frac{16(1-2\epsilon)}{15(1-\lambda)}  d $ (Theorem \ref{theo:lb_12}), we have $ c(P) \geq  \frac{8(1-2\epsilon)}{15(1-\lambda)} d  k $.

\end{proof}

\subsection{Proof on Equivalence between Distributed and Online Algorithms, Lemma \ref{lemma:online-distr-equivalence} }
\begin{proof}
	An $ r $-pass, $ s $-space online algorithm for learning a hypothesis $ h $ on a set $ X $ yields an $ r $-round, $ p $-machine communication protocol for learning $ h(X) $ when $ X $ is partitioned into $ p $ subsets $ S_1,\ldots, S_p $ and the $ i $th machine receives $ S_i $: the $ i $th machine randomly permutes $ S_i $ to generate stream $ s_i $ and the machines emulate the online algorithm on the concatenated stream $ \langle s_1 | s_2 | \ldots |s_p \rangle $. The emulation requires $ O(rps) $ bits of communication.
\end{proof}

\section{Proof on The {\sc Ws} Algorithm, Theorem \ref{theo:alg-ws}}
\begin{proof}
	At the start of each round $ t $, let $ \phi_t $ be the potential function given by the sum of weights of all points in that round. Initially, $ \phi_1 = \sum_{x_i \in D_B}w_i = n $ since by definition for each point $ x_i \in D_B $ we have $ w_i =1 $.
	
	Then in each round, $ A $ constructs a hypothesis $ h_A^t $ at $ B $ to correctly classify the set of points that accounts for at least $ 1-c^\prime = 1-c\frac{1-\lambda}{(1-2\lambda)^2} $ fraction of the total weight by Lemma \ref{lemma:wsample}. All other misclassified points are upweighted by $ (1+\rho) $. Hence, for round $ (t+1) $ we have $ \phi^{t+1} \leq \phi^t((1-c^\prime) + c^\prime(1+\rho )) = \phi^t(1+c^\prime \rho)=n(1+c^\prime \rho)^t $.
	
	Let us consider that weight of the points in the set $ S \subset D_B $ that have been misclassified by a majority of the $ T $ classifiers (after the protocol ends). This implies every point in $ S $ has been misclassified at least $ T/2 $ number of times and at most $ T $ number of times. So the minimum weight of points in $ S $ is $ (1+\rho)^{T/2} $ and the maximum weight is $ (1+\rho)^T $.
	
	Let $ n_i $ be the number of points in $ S $ that has weight $ (1+\rho)^i $ where $ i \in [T/2,T] $. The potential function value of $ S $ after $ T $ rounds is $ \phi_S^T = \sum_{i=T/2}^T n_i (1+\rho)^i $. Our claim is that $ \sum_{i=t}^T n_i = |S| \leq \epsilon n $. Each of these at most $ |S| $ points have a weight of at least $ (1+\rho)^{T/2} $. Hence we have 
	\begin{equation*}
	\phi_S^T = \sum_{i=T/2}^T n_i (1+\rho)^i \geq (1+\rho)^{T/2}\sum_{i=T/2}^T n_i = (1+\rho)^{T/2}|S|.
	\end{equation*}
	Relating these two inequalities we obtain the following,
	\begin{equation*}
	|S|(1+\rho)^{T/2} \leq \phi_S^T \leq \phi^T = n(1+c\rho)^T.
	\end{equation*}
	Hence (using $ T= 5\log_2(1/\epsilon)$)
	\begin{eqnarray*}
		&&|S| \leq n \left(\frac{1+c^\prime \rho}{(1+\rho)^{1/2}}\right)^T \\
		&=& n \left(\frac{1+c^\prime \rho}{(1+\rho)^{1/2}}\right)^{5\log_2(1/\epsilon)} \\
		&=& n(1/\epsilon)^{5\log_2\left(\frac{1+c^\prime \rho}{(1+\rho)^{1/2}}\right)}.
	\end{eqnarray*}
	Setting $ c^\prime = 0.2 $, which means $ c = 0.2(1-2\lambda)^2/(1-\lambda) $ and $ \rho=0.75 $, we get $ 5\log_2\left(\frac{1+c^\prime \rho}{(1+\rho)^{1/2}}\right) < -1 $ and thus $ |S|<n(1/\epsilon)^{-1} < \epsilon n $, as desired since $ \epsilon < 1 $. Thus each round uses $ O(p) $ points, each requiring $ p $ bits of communication, yielding a total communication of $ O(\frac{1-\lambda}{(1-2\lambda)^2}p^2\log(1/\epsilon))  $.
\end{proof}

\section{Upper Bound for Communication Complexity}
We first establish an upper bound on the communication complexity for distributed PAC learning, which is simply derived from  the  sample complexity result for PAC learning with malicious errors \cite{kearns1993learning}. Before stating the sample complexity result, let us recall the {\em Occam algorithm} \cite{blumer1987occam}, whose existence is a sufficient condition for the learnability of a PAC learning problem \cite{valiant1984theory}.
\begin{definition}[Occam Algorithm]
	An algorithm $ A $ is called an Occam algorithm if it draws $ m $ training examples and outputs a hypothesis $ h $ such that $ h $ is consistent with these examples.
\end{definition}
In the presence of malicious outliers (with $ \lambda $ error rate) in the training examples, we desire to develop a $ \lambda $-tolerant algorithm, defined as follows \cite{kearns1993learning}.

\begin{definition}[$ \lambda $-tolerant Algorithm]
	Given a malicious error rate $ 0 \leq \lambda\leq 1/2 $ for the oracles. If an algorithm $ A $ with access to the oracles outputs a hypothesis $ h $ with error probability $ e(h)\leq \epsilon $ with a probability of $ 1-\delta $, the algorithm $ A $ is a $ \lambda $-tolerant algorithm.
\end{definition}
Basically, a $ \lambda $-tolerant algorithm is able to output a hypothesis with bounded error even in presence of a fraction of $ \lambda $ outliers in the training examples. Similarly, we call an Occam algorithm $ A $ as a \emph{$ \lambda $-tolerant Occam algorithm} if $ A $ outputs a hypothesis $ h $ whose error $ e(h) \leq \epsilon$ over $ m $ provided training examples containing malicious outlier with a fraction of $ \lambda $.

Following theorem demonstrates that if we have a $ \lambda $-tolerant Occam algorithm,  we can always develop a $ \lambda $-tolerant  algorithm for learning any target representation with sufficiently many training examples.  
\begin{theorem}[Sample Complexity for Robust PAC Learning]
	\label{theo:cc-upper-bound}
	Suppose an algorithm $ A $ is a $ \lambda $-tolerant Occam algorithm for $ c \in \mathcal{C} $ by $ \mathcal{H} $. Then $ A $ is also a $ \lambda $-tolerant  algorithm for $ c \in \mathcal{C} $ by $ \mathcal{H} $, and the sample size required by $ A $ is $ m = O\left(\frac{(1-\lambda)}{(1-2\lambda)^2} (1/\epsilon \ln 1/\delta + 1/\epsilon \ln |\mathcal{H}| ) \right) $, for achieving an error rate of $ \epsilon $ with a success probability at least $ 1-\delta $.
\end{theorem}
\begin{proof}
	Let $ h \in \mathcal{H}$ be such that its error rate on positive examples $ e^+(h) \geq \epsilon $. Then the probability that $ h $ agrees with a point received from the oracle $ \text{\sc Pos}^\lambda $ is bounded above by
	\begin{equation*}
	(1-\lambda)(1-\epsilon) + \lambda = 1 - (1-\lambda)\epsilon.
	\end{equation*}
	Thus the probability that $ h $ agrees with at least a fraction $ 1-\epsilon/2 $ of $ m $ examples received from  $ \text{\sc Pos}^\lambda $ is bounded above by
	\begin{equation*}
	e^{-m\epsilon (1-2\lambda)^2/8(1-\lambda)},
	\end{equation*}
	by Chernoff bound. From this it follows that the probability that some $ h\in \mathcal{H} $ with $ e^+(h)\geq\epsilon $ agrees with a fraction $ 1-\epsilon/2 $ of the $ m $ examples is at most $ |\mathcal{H}| e^{-m\epsilon (1-2\lambda)^2/8(1-\lambda)} $. Solving $ |\mathcal{H}| e^{-m\epsilon (1-2\lambda)^2/8(1-\lambda)} \leq \delta/2$, we obtain $ m \geq \frac{8}{\epsilon}\frac{(1-\lambda)}{(1-2\lambda)^2}\left(\ln |\mathcal{H}|+\ln \frac{\delta}{2}\right) $.
\end{proof}
Based on the above sample complexity result, a simplest communication protocol is to just have each machine out of $ k $ machines send a random sample of size $ O(\frac{1}{k}\frac{1-\lambda}{(1-2\lambda)^2}(\frac{1}{\epsilon}\ln\frac{1}{\delta}+\frac{1}{\epsilon} \ln |\mathcal{H}|)) $ to a specific machine, which then performs the learning, and there is just one round of communication. Therefore, we can immediately obtain the following upper bound on the communication complexity.
\begin{corollary}[Communication Complexity Upper Bound]
	\label{coro:upper-bd}
	Assume a $ \lambda $-tolerant Occam algorithm always exists.
	Any target representation  $ c \in \mathcal{C} $ can be learned from $ \mathcal{H} $ to error $ \epsilon $ using $ 1 $ round and $ O\left(\frac{(1-\lambda)}{(1-2\lambda)^2}\left(\frac{1}{\epsilon}\ln\frac{1}{\delta}+\frac{1}{\epsilon} \ln |\mathcal{H}|\right) \right)$ total examples communicated. Suppose each example is represented by $ b $ bits, then the total communication complexity is upper bounded by $ O\left(b\frac{(1-\lambda)}{(1-2\lambda)^2}\left(\frac{1}{\epsilon}\ln\frac{1}{\delta}+\frac{1}{\epsilon} \ln |\mathcal{H}|\right)\right) $.
\end{corollary}

\end{document}